\documentclass[11 pt]{article}
\usepackage[utf8]{inputenc}
\usepackage{url}
\usepackage{amsmath}
\usepackage{authblk}
\usepackage{amssymb}
\usepackage{hyperref}
\usepackage{graphics}
\usepackage{mathtools}
\usepackage{caption}
\usepackage{subcaption}

\usepackage{graphicx}
\usepackage{algpseudocode}
\usepackage{algorithm}
\usepackage{afterpage}
\usepackage{array}
\usepackage{collcell}
\usepackage{tcolorbox}
\usepackage{multirow}
\usepackage{amsthm}

\newtheorem{theorem}{Theorem}[section]
\newtheorem{lemma}[theorem]{Lemma}

\newtheorem{definition}[theorem]{Definition}

\usepackage{cite}
\usepackage{booktabs}
\usepackage{appendix}

\usepackage[left=1in,right=1in,top=1.2in,bottom=1.2in,
            footskip=.25in]{geometry}

\def\R{{\mathbb{R}}}
\def\C{{\mathbb C}}
\def\x{{\mathbf x}}
\def\y{{\mathbf y}}

\def\c{{\mathbf c}}

\def\X{{\mathbf X}}
\def\Y{{\mathbf Y}}

 \def\W{{\mathbf{W}}}
           
   \def\A{{\mathbf{A}}}
          
\def\B{{\mathbf{B}}}
   
 \def\S{{\mathcal{S}}}

\title{HARFE: Hard-Ridge Random Feature Expansion}
\author[1]{Esha Saha}
\author[2]{Hayden Schaeffer\footnote{Corresponding Author}}
\author[1]{Giang Tran}

\affil[1]{Department of Applied Mathematics, University of Waterloo}
\affil[2]{Department of Mathematics,
University of California, Los Angeles}

\date{}
\begin{document}

\maketitle

\begin{abstract}
We propose a random feature model for approximating high-dimensional sparse additive functions called the hard-ridge random feature expansion method (HARFE). This method utilizes a hard-thresholding pursuit-based algorithm applied to the sparse ridge regression (SRR) problem to approximate the coefficients with respect to the random feature matrix. The SRR formulation balances between obtaining sparse models that use fewer terms in their representation and ridge-based smoothing that tend to be robust to noise and outliers.  In addition, we use a random sparse connectivity pattern in the random feature matrix to match the additive function assumption.   We prove that the HARFE method is guaranteed to converge with a given error bound depending on the noise and the parameters of the sparse ridge regression model. In addition, we provide a risk bound on the learned model.
Based on numerical results on synthetic data as well as on real datasets, the HARFE approach obtains lower (or comparable) error than other state-of-the-art algorithms. 
\end{abstract}

\section{Introduction}

Kernel-based approaches have been extensively used in data-based applications, including image classification and high-dimensional function approximations since they often perform well in practice. These approaches utilize a pre-defined nonlinear function basis in the form of a kernel $K(\x,\y)$. Using the representer theorem, minimizers of kernel training problems over reproducing kernel Hilbert spaces (RKHS) take the form of linear combinations of the kernel basis applied to the dataset
\cite{smola1998learning,hearst1998support, campbell2002kernel, zhang2005learning}. Since this technique requires one to apply the kernel $K$ to every pair of data samples, the methods scale quadratically with the size of the dataset, which can limit their use in large-scale applications. A more tractable approach utilizes low-dimensional approximations (or factorizations) of the kernel matrix through the use of randomized features.

The random feature model (RFM) \cite{rahimi2007random, rahimi2008uniform, rahimi2008weighted} is a popular technique for approximating the kernel (and thus the minimizer of kernel regression problems) using a randomized basis that can avoid the cost of full kernel methods. If the kernel is positive definite, then it can be written as an inner product with a feature function  $\phi$, i.e. $K(\x,\y) =\langle \phi(\x), \phi(\y) \rangle$. With RFM, a randomized feature map $\psi$ is used to approximate the kernel $K(\x,\y)\approx \psi(\x)^T \psi(\y)$, where $\psi:\R^d \rightarrow \R^N$. When the random feature map is low-dimensional, i.e. when $N$ is not large, then the approximation is tractable for large-scale datasets.  This is the method used for random feature kernel regression.

An alternative perspective is to view the RFM as a nonlinear randomized function approximation. From the neural network point of view, an RFM is a two-layer network with a randomized but fixed single hidden layer \cite{rahimi2007random, rahimi2008uniform, rahimi2008weighted}.  Given an unknown function $f:\R^d \rightarrow\C$, the RFM takes the form $y=\c^T\phi(\W^T\x) \in \C$ where $\x\in \R^{d}$ is the input data, $\W\in \R^{d \times N}$ is a random weight matrix, $\phi$ is a non-linear function applied element-wise and $\c\in\C^{N}$ is the final weight layer. We take as an assumption that the entries of the matrix $\W=[{\boldsymbol{\omega}}_{j,k}]$ are independent and identically distributed (i.i.d.) random variables generated by the (user defined) probability density function $\rho(\boldsymbol{\omega})$ i.e.  ${\boldsymbol{\omega}}_{j,k}\sim\rho(\boldsymbol{\omega})$ for all $1\leq j \leq d$ and $1 \leq k \leq N$. For an RFM, the output layer $\c$ is trained, while the hidden layer $\W$ is fixed. Thus, given a collection of $m$ measurements, whose inputs (and outputs) are arranged column-wise in the matrix $\X\in \R^{d\times m}$ (and $\Y\in \C^{1 \times m}$), the random feature regression problem becomes training $\c$ by optimizing: 
\begin{equation*}
\min_{\c\in\C^{N}}  \ \|\Y-\c^T\phi(\W^T\X)\|_{\ell^2}^2 + \mathcal{R}(\c)
\end{equation*}
with some penalty function $\mathcal{R}:\C^{N} \rightarrow \R$.

The most common choice for $\mathcal{R}$ when using RFM is the ridge penalty $\mathcal{R}(\c) = \lambda_m \, \|\c\|^2_{\ell^2}$ which leads to the random feature ridge regression (RFRR) problem \cite{rahimi2008weighted, rudi2017generalization,li2019towards,weinan2020towards,mei2021generalization}. In \cite{rudi2017generalization}, it was shown that for a function $f$ in an RKHS, using $N=\mathcal{O}({m}^{\frac{1}{2}} \, \log{m})$ number of random features is sufficient to achieve a test error of $\mathcal{O}(m^{-\frac{1}{2}})$ when training the output layer of the RFRR problem. In general, to achieve a risk bound that scales like $\mathcal{O}\left(N^{-1} + m^{-\frac{1}{2}}\right)$ using the RFRR problem over the RKHS,
properties of the spectrum of the kernel operator must be known \cite{weinan2020towards, bach2017equivalence}. In practice, the technical assumptions on the spectrum and its decay may be hard to verify for a given problem or dataset.

When the ridge parameter $\lambda_m\rightarrow 0^+$, the RFRR becomes the min-norm interpolation problem (also referred to as ridge-less regression) and has received recent attention in several different settings \cite{belkin2019reconciling, mei2019generalization, belkin2019does, bartlett2020benign, hastie2019surprises, tsigler2020benign, liang2020just, mei2021generalization}. A detailed analysis of both the kernel ridge regression and the RFRR problems in terms of the dimensional parameters $d$, $m$, and $N$ is provided in \cite{mei2021generalization, chen2021conditioning}, where one conclusion indicates that the min-norm interpolators are the optimal approximations among all kernel methods when $N\gg m$, i.e. in the overparameterized setting. However, noise, outliers, and model misfits often necessitate the use of a penalty in the regression problem.

Since the risk bounds scale is like $N^{-1}$, in order to have a small error, one must choose a large number of features. In addition, it has been observed that the global minimizer of the risk using the RFRR problem as a function of the ratio \,$\frac{N}{m}$\, is achieved for values $\frac{N}{m}\gg 1$  \cite{mei2019generalization, mei2021generalization}, that is, the lower risk solutions occur in the very overparameterized limit. While the risk analysis indicates that large $N$ is needed, the complexity of evaluating an RFM in the overpameterized setting scales linearly with the number of features $N$, but large $N$ can limit the usefulness for  scientific problems or many query applications.  Thus, an alternative approach is to use sparsity-promoting penalties (or algorithms) to obtain sparse or low-complexity models in the overparameterized setting. 

In \cite{hashemi2021generalization}, an $\ell^1$ basis pursuit denoising problem was used to train RFM from limited (and noisy) measurements. Similar to the RFRR problem, it was shown that when the sparsity level $s=N$ and the number of measurements $m= \mathcal{O}(N \log(N))$, the risk is bounded by $\mathcal{O}\left(N^{-1} + m^{-\frac{1}{2}}\right)$ \cite{chen2021conditioning}. When the ``true'' values of the final weight layer $\c$ are compressible, then $s$ can be much smaller than $N$ to achieve similar bounds \cite{hashemi2021generalization,chen2021conditioning}. In \cite{yen2014sparse}, a related algorithm based on the LASSO problem was used to iteratively add a sparse number of random features to the trained RFM. While the $\ell^1$ based approaches lead to good generalization bounds and other theoretical properties, their optimization often comes at a higher computational cost than the ridge regression problem. In \cite{xie2021shrimp}, an iterative pruning approach (related to an $\ell^0$ penalized problem) is shown to perform well on several synthetic and real data examples, in particular, for high-dimensional approximation problems with the inherent low-order structure. They connect the sparsity-promoting methods for RFM to the pruning approaches for reducing the model complexity of overparameterized neural networks. Pruning algorithms focus on obtaining small subnetworks with similar accuracy to the full neural network  \cite{frankle2018lottery, zhou2019deconstructing} and are based on the lotto ticket hypothesis, where the existence of smaller subnetworks with similar (or better) risk is conjectured \cite{frankle2018lottery}. More precisely, the lottery ticket hypothesis said that any randomly initialized, dense neural network (i.e., a random feature network) contains a subnetwork that can match the test accuracy of the original network after training for at most the same number of training iterations. On the other hand, the desired sparsity obtained from the hard thresholding step in our algorithm is essentially a pruning step in order to obtain a subset of features that could lead to better generalization results with a smaller RFM. In this way, we think of thresholding methods as one possible approach to solving the lottery ticket hypothesis.

In this work, we propose the HARFE method, which is a hard ridge-based thresholding algorithm to iteratively obtain a small sub-network for a RFM using a sparsity prior. In addition, we use the sparse random features introduced \cite{hashemi2021generalization}, which restrict each column of the weight matrix $\W$ to have a fixed number of non-zero terms (often only a few non-zero components). This is beneficial when the number of interacting or active variables is low but the ambient dimension is high, see  \cite{harris2019additive, potts2019approximation, potts2021interpretable}. In particular, for sparse additive modeling and decompositions, the sparse random features represent possible interactions between input variables. Additive models for kernel regression and sparse additive models include the multiple kernel learning algorithms \cite{gonen2011multiple, bach2008consistency, xu2010simple}, shrunk additive least squares approximation (SALSA) \cite{ kandasamy2016additive}, sparse shrunk additive models (SSAM) \cite{liu2020sparse}, component selection and smoothing operator (COSSO) \cite{lin2006component},
  additive model with kernel regularization (KAM) \cite{christmann2016learning}, and other related works \cite{potts2021interpretable, tyagi2016learning}.

\subsection{Contributions}
Our main algorithmic and modeling contributions are as follows:
\begin{itemize}
\item We propose the hard-ridge random feature expansion method (HARFE)\footnote{The code is available at \url{https://github.com/esha-saha/HARFE}.}, which is used to train sparse additive random feature models. This method is related to the lotto ticket hypothesis since the desired sparsity obtained from the hard thresholding step in our algorithm is essentially a pruning step in order to obtain a subset of features that could lead to better generalization results with a smaller RFM, see also \cite{frankle2018lottery, zhou2019deconstructing}. 
\item The sparse ridge regression (SRR) problem and some associated (scalable) algorithms have been studied in \cite{luedtke2014branch, mazumder2017subset, bertsimas2020sparse, hazimeh2020fast, xie2020scalable}. The HARFE algorithm is one way to solve the SRR problem and has guarantees on the convergence when the matrix has standard compressive sensing structures. Specifically, the random feature matrix has a coherence bound \cite{hashemi2021generalization} and the restricted isometry property \cite{chen2021conditioning}, thus convergence is given by the results in \cite{iht, foucart2011hard,htp2, maleki2009coherence}. We also obtained a generalization bound for our approach based on the proofs from \cite{hashemi2021generalization, chen2021conditioning}.
\item Tests on synthetic and real-world datasets validate the proposed method, including the inclusion of sparse random features (i.e. for sparse additive random feature modeling). The method performs comparably or better than other related algorithms as shown in the experiments. In addition, the sparsity priors help to obtain important variable dependencies from the data. 
\end{itemize}

\section{Problem Statement and Algorithm}

Throughout the paper, we use bold letters for column vectors (e.g. $\mathbf{x}$) and bold capital letters for matrices (e.g. $\mathbf{A}$). We say that a vector $\mathbf{z}$ is $s$-sparse if it has at most $s$ nonzero entries.  Let $[N]$ denote the set of all positive integers less than or equal to $N$. We denote the $\ell^p$ norm of  a vector $\mathbf{z}$ by $||\mathbf{z}||_p$. Next, we recall a useful definition.

\begin{definition}[\textbf{Order-$q$ Additive Functions} \cite{hashemi2021generalization}]
\label{def:function_class} Fix $d,q,K\in\mathbb{N}$ with $1\leq q\leq d$. A function $f: \mathbb{R}^d \rightarrow \mathbb{C}$ is called an order-$q$ additive function with at most $K$ terms if there exist $K$ complex-valued functions $g_1, \dots, g_K:\R^q\rightarrow\C$ such that 
\begin{align}\label{eq:order-q-def}
f(\x)  =\frac{1}{K}\sum_{j=1}^K g_j(\x|_{\S_j}),
\end{align}
where for each $j\in [K]$, $\S_j\subseteq  [d]$, $\S_j$ has $q$ distinct indices, and $\S_j \neq \S_{j'}$ for $j\not= j'$. Here $\x|_{\S_j}$ denotes the restriction of $\x\in\R^d$ onto $\S_j$. 
\end{definition}
Note that the class of additive functions in Definition \ref{def:function_class} is related to sparse additive modeling and multiple kernel learning \cite{harris2019additive, potts2019approximation, potts2021interpretable, gonen2011multiple, bach2008consistency, xu2010simple, kandasamy2016additive, liu2020sparse, lin2006component, christmann2016learning}.

We are interested in approximating an unknown high dimensional function $f: \R^d\rightarrow \C, d\gg 1,$ from a set of $m$ samples $\{(\mathbf{x}_k,y_k)\}_{k=1}^m$ where the inputs $\x_k$ are drawn from an (unknown) probability measure $\mu(\mathbf{x})$ and the output data is (likely) corrupted by noise:
\begin{equation}
    y_k = f(\mathbf{x}_k) + e_k,\quad \ \text{for}\ k\in [m].
\end{equation}
We assume that the noise $e_k$ is a Gaussian random variable or bounded by some constant $E$, that is,  $|e_k|\leq E \ \forall k \in [m]$. In addition, we assume that the target function $f$ is an order-$q$ additive function with $q\ll d$, $K\ll {d\choose q}$, and that we do not have prior knowledge on the terms $g_j$.
Using the random feature method \cite{rahimi2007random, hashemi2021generalization}, we approximate the target function $f$ by:
\begin{equation*}
f(\mathbf{x}) \approx f^{\#}(\mathbf{x})=\mathbf{c}^T\phi(\W^T\mathbf{x}) =  \sum\limits_{j=1}^N c_j \phi(\langle \x,\boldsymbol{\omega}_j\rangle),
\end{equation*}
where $\mathbf{W}=[\boldsymbol{\omega}_{k,j}]\in \R^{d\times N}$ is a random weight matrix, $\boldsymbol{\omega}_j\in\R^d$ are the column vectors of the matrix $\mathbf{W}$, and $\c\in\C^N$ is the coefficient vector. The random weight matrix $\mathbf{W}\in \R^{d\times N}$ is fixed, while the coefficients $\c\in\C^N$ are trainable. The function $\phi: \R\rightarrow\R$ is the nonlinear activation function and can be chosen to be a trigonometric function, the sigmoid function, or the ReLU function. Unless otherwise stated, we use the sine activation function, i.e. $\phi(\cdot)=\sin(\cdot)$.
This model is a two-layer neural network with the weights in the hidden layer being randomized but not trainable and thus the training problem relies on learning the coefficient vector $\c$. Theoretically, the random feature method has been shown to be comparable with shallow networks in terms of theoretical risk bounds \cite{rahimi2007random, rahimi2008weighted, rahimi2008uniform, rudi2017generalization} where the population risk is defined as
\begin{equation}
    R(f^\#) := ||f-f^\#||_{L^2(d\mu)}^2 = \int\limits_{\mathbb{R}^d}|f(\x)-f^\#(\x)|^2d\mu(\x)
\end{equation}
which is also the $L^2$ squared error between the
true function and its approximation.
Suppose the entries of the random weight matrix $\W$  are i.i.d. random variables generated by the (one-dimensional) probability function $\rho(\boldsymbol{\omega})$, $\boldsymbol{\omega}_{k,j}\sim \rho(\boldsymbol{\omega})$. Let $\mathbf{A}\in\C^{m\times N}$ be the random feature matrix whose entries are defined as $a_{k,j} = \phi(\langle \x_k, \mathbf{\boldsymbol{\omega}}_j\rangle)$. Then the general regression problem is to solve the following optimization problem:
\begin{equation}
\min\limits_{\mathbf{c}\in\mathbb{C}^N}||\A\mathbf{c}-\mathbf{y}||_2^2,
\end{equation}
where $\mathbf{y} = [y_1,y_2,...,y_m]^T$.

For sparse additive modeling, since $f$ is an order-$q$ function, each entry of the random feature matrix $\mathbf{A}$ should only depend on $q$ entries of the input data $\x_k$. Therefore, we can instead generate a sparse random matrix $\mathbf{W}$ where each column of $\mathbf{W}$ has at most $q$ nonzero entries following the probability function $\rho(\boldsymbol{\omega})$ \cite{hashemi2021generalization}. One such way to generate $\mathbf{W}$ is to first generate $N$ random vectors $\mathbf{v}^{(j)}=(v_{1}^{(j)}, \ldots, v_{q}^{(j)})^T$ in $\R^q$ and use a random embedding that assigns $\mathbf{v}^{(j)}$ to $\boldsymbol{\omega}_j$, where $\boldsymbol{\omega}_j =(0,0,\ldots, 0, v_{1}^{(j)},0,\ldots, v_{2}^{(j)},0,\ldots, v_{q}^{(j)},0,\ldots ,0)^T$. In particular, for each $j$, we select a subset of $q$ indices from {$[d]$} uniformly at random and then sample each nonzero entry using $\rho(\boldsymbol{\omega})$. The sparse random matrix $\mathbf{W}$ can also be obtained as $\mathbf{W} = \widetilde{\mathbf{W}}\odot\mathbf{M}$, where $\widetilde{\mathbf{W}}\in\mathbb{R}^{d\times N}$ is a dense matrix whose entries are sampled from $\rho(\boldsymbol{\omega})$, the mask $\mathbf{M}\in\mathbb{R}^{d\times N}$ is a sparse matrix whose non-zero entries are one and each column of $\mathbf{M}$ has $q$ non-zero entries, and $\odot$ denotes the element-wise multiplication. Using this formulation, the general sparse regression problem becomes
\begin{equation}
    \text{find } \mathbf{c}\in \C^N\  \text{such that } ||\A\mathbf{c} - \mathbf{y}||_2\leq\epsilon\sqrt{m} \quad \text{and}\quad \c \quad \text{is sparse},
\end{equation}
where $\epsilon$ is the parameter related to the noise level. The
motivation for sparsity in $\c$ is due to the assumption that  $K\ll {d\choose q}$, thus not all index subsets are needed. 

 In order to solve the sparse random feature regression problem, we propose a new greedy algorithm named hard-ridge random feature expansion (HARFE), which uses a hard thresholding pursuit (HTP) like algorithm to solve the random feature ridge regression problem. Specifically, we learn $\c$ from the following minimization problem:
\begin{equation}\label{eqn:htp-r}
    \min\limits_{\mathbf{c}\in\C^N} \|\A\mathbf{c} - \mathbf{y}\|_2^2 + m \lambda ||\mathbf{c}||_2^2\quad \text{such that\,}\,\,\mathbf{c}\,\,\text{is $s$-sparse},
\end{equation}
where $\lambda>0$ is the regularization parameter. Equation \eqref{eqn:htp-r} can be rewritten as
\begin{equation}\label{eqn:modifiedHTP}
\min\limits_{\textbf{c}\in\C^N}\|\mathbf{B}\mathbf{c} - \tilde{\mathbf{y}}||_2^2 \quad \text{such that\,}\,\, \mathbf{c} \,\,\text{\,is $s$-sparse},
\end{equation}
where
$\mathbf{B} = 
\begin{bmatrix}
\A\\
\sqrt{m\lambda}\mathbf{I}_N
\end{bmatrix}\in \C^{(m+N)\times N}$ 
and $\tilde{\mathbf{y}} =
\begin{bmatrix}
\mathbf{y}\\
\mathbf{0}
\end{bmatrix}\in\C^{m+N}$. To solve \eqref{eqn:modifiedHTP}, we first start with an $s$-sparse vector $\mathbf{c}^0\in \C^N$ (typically taken as $\mathbf{c}^0 = 0$) and update based on the HTP approach
\begin{equation}
\begin{split}
    & S^{n+1} = \{ \text{indices of }s\  \text{largest (in magnitude) entries of }\mathbf{c}^n + \mu \mathbf{B}^*(\mathbf{y} - \mathbf{B}\mathbf{c}^n)\}\\
     & \mathbf{c}^{n+1} = \text{argmin}\{\|\tilde{\mathbf{y}} - \mathbf{B}\mathbf{c}\|_2^2,\quad  \text{supp}(\mathbf{c})\subseteq S^{n+1}\}
     \end{split}\label{HTP2}
\end{equation}
where $\mu>0$ is the step-size and $s$ is a user defined parameter. The idea is to solve for the coefficients using a much smaller number of model terms. The subset $S$ given by the indices of the $s$ largest entries of one gradient descent step applied on the vector $\mathbf{c}$ is a good candidate for the support set of $\mathbf{c}$. The HTP algorithm iterates between these two steps and leads to a stable and robust reconstruction of sparse vectors depending on the RIP constant. The Gram matrix $\mathbf{B}^*\mathbf{B}$ is computed directly based on the ridge problem
\[
\mathbf{c}^n + \mu \mathbf{B}^*(\tilde{\mathbf{y}} - \mathbf{B}\mathbf{c}^n) =(1-m\mu\lambda)\mathbf{c}^n + \mu \A^*(\mathbf{y} - \A\mathbf{c}^n).
\]
The approach is summarized in Algorithm \ref{alg:cap}. The relative residual at the iterate $\mathbf{c}^n$ is  defined as  \[\text{Relative\ Residual} = \dfrac{||\A\c^n - \y||_2}{||\y||_2}.\]
\begin{figure}[t!]
    \centering
    \includegraphics[scale =0.6]{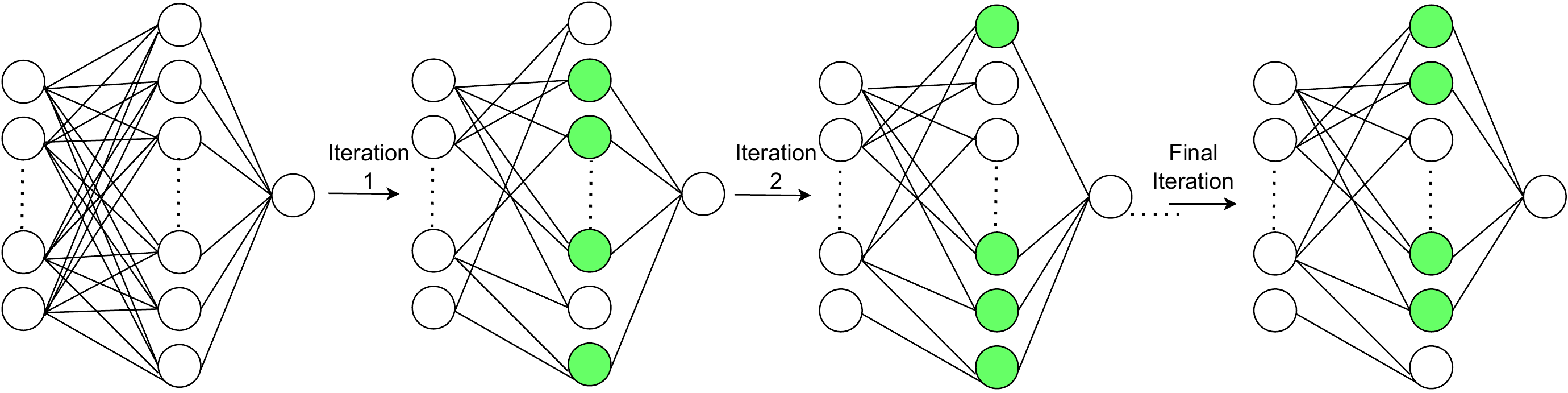}
    \caption{Schematic representation of HARFE. The active nodes at each iteration is given in green.}
    \label{fig:schematic}
\end{figure}

\begin{algorithm}[t!]
\caption{Hard-Ridge Random Feature Expansion (HARFE)}\label{alg:cap}
\begin{algorithmic}[t]
\Require Samples $\{\mathbf{x}_k,y_k\}_{k=1}^m$, non-linear function $\phi$, number of features $N$, sparsity level $s$, {random weight sparsity $q$}, step size $\mu$, regularization parameter $\lambda$, convergence threshold $\epsilon$ and total number of iterations \texttt{tot\_iter}.
\State Draw ($q$-sparse) $N$ random weights $\boldsymbol{\omega}_j$ whose non-zero entries are sampled from $\rho(\boldsymbol{\omega})$. 
\State Construct the random feature matrix $\mathbf{A}= [\phi(\langle\mathbf{x}_k;\boldsymbol{\omega}_j\rangle)]\in\C^{m\times N}$.

\Ensure 

\State {\bf Initialization:} Start with $s$-sparse $\mathbf{c}^0\in \C^N$ ($\mathbf{c}^0 = 0$), $n=0$
\While{(Relative Residual$>\epsilon)$ or \texttt{(n<tot\_iter)}}
    \State $\Tilde{\mathbf{c}}^{n+1} \gets (1-m\mu\lambda)\mathbf{c}^n + \mu \A^*(\mathbf{y} - \A\mathbf{c}^n)$
    \State $\text{idx} \gets \text{indices of } s \text{ largest entries of } \Tilde{\mathbf{c}}^{n+1} $  \Comment{Choose the subset of features $S^{n+1}$}
    \State $\Bar{\A} \gets \A[:,\text{idx}]$
    \State $\mathbf{c}^{n+1}[[N]\setminus\text{idx}] =0$
   \State $\mathbf{c}^{n+1}[\text{idx}] \gets \text{argmin}\{||\mathbf{y} - \bar{\A}\mathbf{c}||_2^2 + m\lambda||\mathbf{c}||_2^2\} = (\bar{\A}^* \bar{\A} + m\lambda \mathbf{I}_s)^{-1} \bar{\A}^*\mathbf{y}$
    
     \State $n=n+1$

\EndWhile\\
\Return Sparse vector $\c = [c_1,c_2,...,c_N]$ such that:    $f^{\sharp}(\mathbf{x}) =\sum\limits_{j=1}^N c_j\phi(\mathbf{x};\boldsymbol{\omega}_j)$.

\end{algorithmic}

\end{algorithm}

One of the motivations for including the ridge penalty is that for random feature regression, the sparsity level $s$ can be large and thus the least squares step in \eqref{HTP2} may be ill-conditioned. Numerically, we observed that even a small non-zero value of $\lambda$ can be beneficial for ensuring convergence and good generalization. In Figure~\ref{fig:schematic}, we provide a schematic representation of the sequence generated by the algorithm, namely, over each step a sparse subset of nodes is obtained until a final configuration is achieved.

\section{Theoretical Discussion}

The error produced by the HARFE algorithm can be established by extending the results on the HTP algorithm to include the ridge regression term and by leveraging bounds on the restricted isometry constant for this type of random feature matrix. Recall that given an integer $s \in [N]$, the $s$-th restricted isometry constant of a matrix $\mathbf{A}\in \mathbb{C}^{m\times N}$, denoted by $\delta_s(\mathbf{A})$,  is the smallest non-negative $\delta$ such that
\begin{align*}
(1-\delta)\|\mathbf{x}\|_2^2 \leq \|\mathbf{A}\mathbf{x}\|_2^2 \leq (1+\delta)\|\mathbf{x}\|_2^2
\end{align*}
holds for all $s$-sparse $\mathbf{x} \in \mathbb{C}^N$ \cite{foucart2013mathematical}. Let $\kappa_{1,s}(\mathbf{x})$ denote the $\ell^1$ distance to the best $s$-term approximation of $\mathbf{x}$ defined by
\begin{align*}
\kappa_{1,s}(\mathbf{x})= \inf  \left\{\|\mathbf{z}-\mathbf{x}\|_1: {\mathbf{z}}\in \C^n,\, \mathbf{z} \ \text{is} \ s-\text{sparse}\right\}.
\end{align*}
The value $\kappa_{1,s}(\mathbf{x})$ provides a measure for the compressibility of the vector $\mathbf{x}$ with respect to the $\ell^1$ norm and is obtained by setting all but the $s$-largest in magnitude entries to zero. 

The following result is restricted to the case when $q=d$ and $\mu = 1$. The $q\leq d$ case follows from similar arguments \cite{hashemi2021generalization}. From \cite{foucart2013mathematical}, the theorem below can be extended trivially for any step size $\mu$ by scaling the matrix $\mathbf{A}$ and the vector $\mathbf{y}$ with the ratio $ \sqrt{\mu}$.

\begin{theorem}[Convergence of the iterates of HARFE]\label{thm:convergence}
Let the data $\{\mathbf{x}_k\}_{k\in [m]}$ be drawn from $\mathcal{N}(\mathbf{0},\gamma^2 \mathbf{I}_d)$, the weights $\{\boldsymbol{\omega}_j\}_{j\in [N]}$ be drawn from $\mathcal{N}(\mathbf{0},\sigma^2 \mathbf{I}_d)$, and the random feature matrix $\A\in\mathbb{C}^{m\times N}$ be defined component-wise by $a_{k,j} = \exp(i \langle \mathbf{x}_k,{\boldsymbol{\omega}}_j \rangle)$. Denote the $\ell_2$-regularization parameter by $\lambda$ and the sparsity level by $s$. If 
\begin{align*}
    m&\geq C_1\, (1+\lambda)^{-2}\, s\, \log(\delta^{-1}), \\
    \frac{m}{\log(3m)} &\geq C_2 \, (1+\lambda)^{-1}\, s \,\log^2(6s) \log\left(\frac{N}{9\log(2m)}+3\right), \\
     \frac{\sqrt{\delta}}{6\sqrt{3}}&\, (1+\lambda)\,  (4\gamma^2\sigma^2+1)^{\frac{d}{4}}\geq N,
\end{align*}
where $C_1$ and $C_2$ are universal positive constants, then with probability at least $1-2\delta$, for all $\mathbf{c}\in \mathbb{C}^N$ and $\mathbf{e}\in \mathbb{C}^m$ with $\mathbf{y}=\A\mathbf{c}+\mathbf{e}$, the sequence $\mathbf{c}^{n}$ defined by the Algorithm~\ref{alg:cap} with $\mathbf{c}^0=\mathbf{0}$, using $2s$ instead of $s$ in the algorithm, satisfies
\begin{align*}
\| \mathbf{c}^{n}-\mathbf{c}\|_2 &\leq 2\beta^n  \|\mathbf{c}\|_2 + \frac{D_1}{\sqrt{s}} \kappa_{1,s}(\mathbf{c})+D_2 \, \sqrt{\frac{m^{-1}\|\mathbf{y}-\A\mathbf{c}\|^2_2+\lambda \|\mathbf{c}\|^2_2}{1+ \lambda}}
\end{align*}
for all $n\geq 0$ where the constants $\beta \in(0,1), D_1,D_2>0$ depend only on $\delta_{6s}(\B)$. The matrix $\B$ is given by $\B=\left(m+ m\lambda \right)^{-\frac{1}{2}}
\begin{bmatrix}
\A\\
\sqrt{m\lambda}I_N
\end{bmatrix}\in \mathbb{C}^{(m+N)\times N}$.
\end{theorem}

\label{sec:proof}
\begin{proof}
The ridge regression problem:
\begin{equation} \label{eqn:regression_prime}
\min\limits_{\textbf{c}\in\mathbb{R}^N} \|\A\mathbf{c} - \mathbf{y}\|_2^2+ m\lambda \|\mathbf{c}\|_2^2
\end{equation}
can be written in the form:
\begin{equation} \label{eqn:regression_rescaled}
\min\limits_{\textbf{z}'\in\mathbb{R}^N}\|\B\mathbf{z'} - \tilde{\mathbf{y}}\|_2^2 
\end{equation}
where
$\B = \left(m+ m\lambda \right)^{-\frac{1}{2}}
\begin{bmatrix}
\A\\
\sqrt{m\lambda}I_N
\end{bmatrix}\in \mathbb{C}^{(m+N)\times N}$ 
and $\tilde{\mathbf{y}} =
\begin{bmatrix}
\mathbf{y}\\
\mathbf{0}
\end{bmatrix}\in\mathbb{C}^{m+N}$. For equation \eqref{eqn:regression_rescaled}, the error is defined as $\tilde{\mathbf{e}}=\tilde{\mathbf{y}}-\B\textcolor{green}{\mathbf{z}'}$ which is equivalent to $$\tilde{\mathbf{e}} =
\begin{bmatrix}
\mathbf{e}\\
-\sqrt{m\lambda}\, \mathbf{c}
\end{bmatrix}\in\mathbb{C}^{m+N}.$$
If $\textbf{c}$ is the minimizer of \eqref{eqn:regression_prime} and $\textbf{z}'$ is  the minimizer of \eqref{eqn:regression_rescaled}, then $\textbf{z}'= \left(m+ m\lambda \right)^{\frac{1}{2}} \textbf{c}$. The scaling is needed for the matrix to satisfy a restricted isometry property. 

To estimate the restricted isometry constant of $\B$, we first bound the restricted isometry constant of $m^{-\frac{1}{2}}\A$. By Theorem~\ref{thm:RIP} and the assumptions, if 
\begin{align*}
    m&\geq 6C_1 \eta_1^{-2}\, s\, \log(\delta^{-1}) \\
    \frac{m}{\log(3m)} &\geq 6C_2 \eta_2^{-2}\, s \,\log^2(6s) \log\left(\frac{N}{9\log(2m)}+3\right) \\
    \sqrt{\delta}&\,\eta_1\, (4\gamma^2\sigma^2+1)^{\frac{d}{4}}\geq N,
\end{align*}
where $C_1$ and $C_2$ are universal positive constants, then with probability at least $1-2\delta$, the $6s$-restricted isometry constant is bounded by
$$
\delta_{6s}\left(\A\right)< 3\eta_1 + \eta_2^2 + \sqrt{2}\eta_2,
$$
or equivalently
$$\left\|m^{-1} \A_S^* \A_S - \mathbf{I}_S\right\|_{2\rightarrow 2}  <3\eta_1 + \eta_2^2 + \sqrt{2}\eta_2,$$
for all $S\subset [N]$ with $|S|=6s$. Therefore, the $6s$-restricted isometry constant of $B$ satisfies
\begin{align*}
\left\| \B_S^* \B_S - \mathbf{I}_S\right\|_{2\rightarrow 2}  &=\left\| \left(m+ m\lambda \right)^{-1} \left(\A_S^* \A_S+m \lambda \mathbf{I}_S\right) - \mathbf{I}_S\right\|_{2\rightarrow 2}\\
&=\frac{1}{1+\lambda} \left\| m^{-1} \A_S^* \A_S - \mathbf{I}_S\right\|_{2\rightarrow 2}\\
&<\frac{3\eta_1 + \eta_2^2 + \sqrt{2}\eta_2}{1+\lambda}.
\end{align*}
Setting the parameters to $\eta_1=\frac{1+\lambda}{6\sqrt{3}}$ and $\eta_2=\frac{\sqrt{1+\lambda}}{4\sqrt{3}}$, then
$\left\| \B_S^* \B_S - \mathbf{I}_S\right\|_{2\rightarrow 2}<\frac{1}{\sqrt{3}}$
if 
\begin{align*}
    m&\geq 648C_1\, (1+\lambda)^{-2}\, s\, \log(\delta^{-1}) \\
    \frac{m}{\log(3m)} &\geq 288C_2 \, (1+\lambda)^{-1}\, s \,\log^2(6s) \log\left(\frac{N}{9\log(2m)}+3\right) \\
     \frac{\sqrt{\delta}}{6\sqrt{3}}&\, (1+\lambda)\,  (4\gamma^2\sigma^2+1)^{\frac{d}{4}}\geq N.
\end{align*} 
By Equation (\ref{eq:convhtpseq}) of Theorem~\ref{thm:convergenceHTP}, the sequence generated by the hard thresholding pursuit algorithm produces a solution with the following bound
\begin{align}
\| \mathbf{z'}^{n}-\mathbf{z'}\|_2 &\leq 2\beta^n \|\mathbf{z'}\|_2 + \frac{C}{\sqrt{s}} \kappa_{1,s}(\mathbf{z'})+D \|\tilde{\mathbf{e}}\|_2,
\end{align}

for all $n\geq 0$ where the constants $\beta \in(0,1), C,D>0$ depend only on $\delta_{6s}(\B)$. Transforming the variables to the original variables yields the following bound
\begin{align}
\| \mathbf{c}^{n}-\mathbf{c}\|_2 &\leq 2\beta^n \|\mathbf{c}\|_2 + \frac{C}{\sqrt{s}} \kappa_{1,s}(\mathbf{c})+D\,  \left(1+ \lambda \right)^{-\frac{1}{2}} \, m^{-\frac{1}{2}} \, \|\tilde{\mathbf{e}}\|_2\\
&\leq 2\beta^n \|\mathbf{c}\|_2 + \frac{C}{\sqrt{s}} \kappa_{1,s}(\mathbf{c})+D\,  \left(1+ \lambda \right)^{-\frac{1}{2}} \, m^{-\frac{1}{2}} \, \sqrt{\|\mathbf{e}\|^2_2+m\lambda \|\mathbf{c}\|^2_2}\\
 &\leq 2\beta^n  \|\mathbf{c}\|_2 + \frac{C}{\sqrt{s}} \kappa_{1,s}(\mathbf{c})+D \, \sqrt{\frac{m^{-1}\|\mathbf{e}\|^2_2+\lambda \|\mathbf{c}\|^2_2}{1+ \lambda}},
\end{align}
which concludes the proof.\end{proof}

It is worth noting that, in practice, $\lambda>0$ is small, i.e. we would like $m\lambda = \mathcal{O}(1)$. Thus the third term in the iterative bound is smaller than the second term and does not contribute significantly to the overall error.

\begin{theorem}[Risk bound for HARFE]\label{thm:risk}
Let the data $\{\mathbf{x}_k\}_{k\in [m]}$ be drawn from $\mathcal{N}(\mathbf{0},\gamma^2 \mathbf{I}_d)$, the weights $\{\boldsymbol{\omega}_j\}_{j\in [N]}$ be drawn from $\mathcal{N}(\mathbf{0},\sigma^2 \mathbf{I}_d)$, and the random feature matrix $\A\in\mathbb{C}^{m\times N}$ be defined component-wise by $a_{k,j} = \exp(i \langle \mathbf{x}_k,{\boldsymbol{\omega}}_j \rangle)$. Denote the $\ell_2$ regularization parameter by $\lambda$, accuracy parameter by $\epsilon>0$, and the sparsity level by $s$. If the assumptions of Theorem \ref{thm:convergence} are satisfied, then with probability at least $1-2\delta$, the following risk bounds hold:
\begin{align*}
&R(f^\#) \leq \|f\|_{\rho}\left(\epsilon +  D\sqrt{\dfrac{\lambda}{N}}\left(3^{-\frac{1}{4}} + \left(2m\log\left(\dfrac{1}{\delta}\right)\right)^{\frac{1}{4}}\right)\right)\\  & \ \ +C\sqrt{\dfrac{1+\lambda}{s}}\left(3^{-\frac{1}{4}} + N^{\frac{1}{2}}\left(2m\log\left(\dfrac{1}{\delta}\right)\right)^{\frac{1}{4}}\right)\kappa_{s,1}(\|\c^{\star}\|)
    + DE\left(3^{-\frac{1}{4}}m^{-\frac{1}{2}} + N^{\frac{1}{2}}\left(\dfrac{2}{m}\log\left(\dfrac{1}{\delta}\right)\right)^{\frac{1}{4}}\right).
\end{align*}
\end{theorem}
\begin{proof}
We use the $L^2$ norm, which can be decomposed into two parts using the triangle inequality:
\[||f-f^{\sharp}||_{L^2(d\mu)} \leq ||f-f^{\star}||_{L^2(d\mu)} +||f^{\star}-f^{\sharp}||_{L^2(d\mu)}.\]
The approximation $f^{\sharp}$ is defined in equation (\ref{eq:linearcombo}) and the best $\phi$-based approximation $f^{\star}$  is given in equation (\ref{eq:fstar}).

Following the proof in Section 6 of \cite{chen2021conditioning}, if $N\geq \dfrac{1}{\epsilon^2}\left(1+\sqrt{2\log\left(\dfrac{1}{\delta}\right)}\right)^2$, then with probability at least $1-\delta$, we have
\begin{equation}\label{eq:lemma1}
||f-f^{\star}||_{L^2(d\mu)}\leq \epsilon ||f||_{\rho}.
\end{equation}

We use McDiarmid's Inequality to bound $||f^{\star}-f^{\sharp}||_{L^2(d\mu)}$, arguing similarly as in Lemma 2 from \cite{hashemi2021generalization} or Section 6 of \cite{chen2021conditioning}. Let $\{\mathbf{z}_j\}_{j\in[m]}$ be i.i.d. random variables sampled from the distribution $\mu$ which are independent from the sampled points $\{\mathbf{x}_j\}_{j\in [m]}$ and the frequencies $\{\boldsymbol{\omega}\}_{k\in[N]}$. This independence assumption makes $\{z_j\}_{j\in [m]}$ also independent of coefficients $\c^{\sharp}$ and $\c^{\star}$ and thus allows for the following argument. Let $v$ be a random variable defined by

\[v(\mathbf{z_1,...,z_m}) = \|f^{\star}-f^{\sharp}\|_{L^2(d\mu)}^2 - \dfrac{1}{m}\sum_{j=1}^m |f^{\star}(\mathbf{z}_j) - f^{\sharp}(\mathbf{z}_j)|^2.\]
Then $\mathbb{E}_{\mathbf{z}}[v] = 0$ as
\[ \mathbb{E}_{\mathbf{z}}[|f^{\star}(\mathbf{z}_j) - f^{\sharp}(\mathbf{z}_j)|^2] = \mathbb{E}_{\mathbf{z_1,...,z_m}}[|f^{\star}(\mathbf{z}_j) - f^{\sharp}(\mathbf{z}_j)|^2] = \|f^{\star}-f^{\sharp}\|_{L^2(d\mu)}^2.\]
We perturb the $k-$th component of $v$ to get,
\[|v(\mathbf{z_1,...,z_k,...,z_m}) - v(\mathbf{z_1,...,\tilde{z}_k,...,z_m})|\leq \dfrac{1}{m}\left||f^{\star}(\mathbf{z}_k) - f^{\sharp}(\mathbf{z}_k)|^2 - |f^{\star}(\mathbf{\tilde{z}}_k) - f^{\sharp}(\mathbf{\tilde{z}}_k)|^2\right|.\]
Using Cauchy-Schwarz inequality, for any $\mathbf{z}$, we have,
\[|f^{\star}(\mathbf{\tilde{z}}_k) - f^{\sharp}(\mathbf{\tilde{z}}_k)|^2 \leq N\| \tilde{\c}^{\star} - \tilde{\c}^{\sharp}\|_2^2,\]
which holds since $|\phi(\mathbf{z},\boldsymbol{\omega})|=1.$ Hence,
\[|v(\mathbf{z_1,...,z_k,...,z_m}) - v(\mathbf{z_1,...,\tilde{z}_k,...,z_m})|\leq \dfrac{2N}{m}\| \tilde{\c}^{\star} - \tilde{\c}^{\sharp}\|_2^2:=\Delta.\]

Therefore, we can apply McDiarmid’s inequality to the random variable $v$, i.e. $P_{\mathbf{z}}(v-\mathbb{E}_{\mathbf{z}}[v] \geq t) \leq \exp(-\frac{2t^2}{m\Delta^2})$ where $t := \Delta\sqrt{\dfrac{m}{2}\log\left(\dfrac{1}{\delta}\right)}$. Following the results from Theorem \ref{thm:convergence}, we have that $\delta_{6s}(\B)<\frac{1}{\sqrt{3}}$ (the matrix $\B$ is as obtained in equation (\ref{eqn:regression_rescaled})), then with $s$ replaced by $2s$, with probability at least $1-3\delta$ ($2\delta$ for the coherence bound and $\delta$ from the \ref{eq:lemma1}), we have:
\begin{center}
$\|f^{\star}-f^{\sharp}\|_{L^2(d\mu)}^2 \leq \dfrac{1}{m}\sum_{j=1}^m |f^{\star}(\mathbf{z}_j) - f^{\sharp}(\mathbf{z}_j)|^2 + N\left(\sqrt{\dfrac{2}{m}\log\left(\dfrac{1}{\delta}\right)}\right)\| \tilde{\c}^{\star} - \tilde{\c}^{\sharp}\|_2^2$\\
$=\dfrac{1}{m}\|\tilde{\mathbf{B}}(\tilde{\c}^{\star} - \tilde{\c}^{\sharp})\|_2^2 + N\left(\sqrt{\dfrac{2}{m}\log\left(\dfrac{1}{\delta}\right)}\right)\| \tilde{\c}^{\star} - \tilde{\c}^{\sharp}\|_2^2$\\
$\leq \dfrac{1}{\sqrt{3}m}\|(\tilde{\c}^{\star} - \tilde{\c}^{\sharp})\|_2^2 + N\left(\sqrt{\dfrac{2}{m}\log\left(\dfrac{1}{\delta}\right)}\right)\| \tilde{\c}^{\star} - \tilde{\c}^{\sharp}\|_2^2.$
\end{center}
From Equation (\ref{eq:convhtpcluster}) of Theorem~\ref{thm:convergenceHTP}, $\|(\tilde{\c}^{\star} - \tilde{\c}^{\sharp})\|_2\leq \dfrac{C}{\sqrt{s}}\kappa_{s,1}(\tilde{\c}^{\star}) + D\|\tilde{\mathbf{e}}\|_2$, where $C,D>0$ depend only on $\delta_{6s}(\B)$. Since $\tilde{\mathbf{c}} = \sqrt{m+m\lambda}\c$, transforming to original variables, we have:
\begin{equation}\label{eq:lemma2}
    \|f^{\star}-f^{\sharp}\|_{L^2(d\mu)} \leq \left(\dfrac{1}{\sqrt{3}m} + N\left(\sqrt{\dfrac{2}{m}\log\left(\dfrac{1}{\delta}\right)}\right)\right)^{\frac{1}{2}}\left(\sqrt{m+m\lambda}\dfrac{C}{\sqrt{s}}\kappa_{s,1}(\|\c^{\star}\|) + D\sqrt{\|\mathbf{e}\|^2 + m\lambda\|\c^{\star}\|_2^2}\right)
\end{equation}
Note $|\c_k^{\star}|= \dfrac{1}{N}\left|\dfrac{\alpha(\boldsymbol{\omega}_k)}{\rho(\omega)_k}\right|\leq \dfrac{1}{N}\|f\|_{\rho}$ and $\|\mathbf{e}\|_2\leq E$. Thus, combining Equations (\ref{eq:lemma1}) and (\ref{eq:lemma2}) yields
\begin{align*}
    &\|f - f^{\sharp}\|_{L^2(d\mu)} \\
    &\leq \epsilon\|f\|_{\rho} + \left(\dfrac{1}{\sqrt{3}m} + N\left(\sqrt{\dfrac{2}{m}\log\left(\dfrac{1}{\delta}\right)}\right)\right)^{\frac{1}{2}}\left(\sqrt{m+m\lambda}\dfrac{C}{\sqrt{s}}\kappa_{s,1}(\|\c^{\star}\|) + D\sqrt{E^2 + m\lambda N^{-1}\|f\|_{\rho}^2}\right)\\
    &\leq \epsilon\|f\|_{\rho} + \left(3^{-\frac{1}{4}}m^{-\frac{1}{2}} + N^{\frac{1}{2}}\left(\dfrac{2}{m}\log\left(\dfrac{1}{\delta}\right)\right)^{\frac{1}{4}}\right)\left(\sqrt{m}\sqrt{1+\lambda}\dfrac{C}{\sqrt{s}}\kappa_{s,1}(\|\c^{\star}\|) + D(E + \sqrt{m\lambda} N^{-\frac{1}{2}}\|f\|_{\rho}\right)\\
    & = \|f\|_{\rho}\left(\epsilon +  D\sqrt{\dfrac{\lambda}{N}}\left(3^{-\frac{1}{4}} + \left(2m\log\left(\dfrac{1}{\delta}\right)\right)^{\frac{1}{4}}\right)\right)\\  &\ \ +C\sqrt{\dfrac{1+\lambda}{s}}\left(3^{-\frac{1}{4}} + N^{\frac{1}{2}}\left(2m\log\left(\dfrac{1}{\delta}\right)\right)^{\frac{1}{4}}\right)\kappa_{s,1}(\|\c^{\star}\|)
    + DE\left(3^{-\frac{1}{4}}m^{-\frac{1}{2}} + N^{\frac{1}{2}}\left(\dfrac{2}{m}\log\left(\dfrac{1}{\delta}\right)\right)^{\frac{1}{4}}\right).
\end{align*}
\end{proof}

 Theorem~\ref{thm:convergence} highlights a theoretical purpose of $\lambda>0$ in terms of convergence. The constants $D_1,D_2>0$ in Theorem~\ref{thm:convergence} depend on $\delta_{6s}\left(\B\right)$.  As $\lambda$ approaches zero, the value of $\delta_{6s}\left(\B\right)$ approaches the larger value of $\delta_{6s}\left(A\right)$ and thus $\beta$ increases, which can lead to slower convergence. As $\lambda$ becomes large, the solution approaches zero and the error bounds in Theorem~\ref{thm:convergence}  become trivial. In practice, we found that a small non-zero value is useful for convergence and for mitigating the effects of noise and outliers.

Theorem~\ref{thm:convergence} is stated for any vector $\mathbf{c}$. We can consider two potential vectors $\mathbf{c}$ depending on the scaling of $N$ and $m$. First, if $N$ is sufficiently large, then by the results of \cite{hashemi2021generalization, chen2021conditioning} the matrix $\A$ will be well-conditioned with high probability (for any $\lambda>0$) and thus there exists a $\mathbf{c}$ such that $\mathbf{e}=\mathbf{0}$. Therefore, for small $\lambda>0$ the relative error is dominated by the compressibility of $\mathbf{c}$:
\begin{align}
\frac{\| \mathbf{c}^{n}-\mathbf{c}\|_2}{ \|\mathbf{c}\|_2} \leq 2\beta^n   + \frac{C}{\sqrt{s}} \frac{\kappa_{1,s}(\mathbf{c})}{\|\mathbf{c}\|_2}+D \, \sqrt{\frac{\lambda }{1+ \lambda}}. 
\end{align}
In this setting, the HARFE algorithm can be seen as a pruning approach that generates a subnetwork with $s$ connections that is an approximation to the full $N$-parameter network, see \cite{frankle2018lottery, zhou2019deconstructing}.

Alternatively, we can consider the function approximation results found in \cite{rahimi2007random, rahimi2008weighted, rahimi2008uniform, hashemi2021generalization}. Suppose we are given a probability density $\rho$ used to sample the entries of the random weights $\boldsymbol{\omega} \in \mathbb{R}^d$ and a function $\phi:\mathbb{R}^{2d} \to \mathbb{C}$. A function $ f\in  \mathcal{F}(\phi,\rho) $ if $f:\mathbb{R}^d \to \mathbb{C}$ has finite $\rho$-norm with respect to $\phi$ defined by
\begin{equation*}
    \mathcal{F}(\phi,\rho)=\left\{ f(\mathbf{x}) = \int_{\mathbb{R}^d} \alpha(\boldsymbol{\omega}) \phi(\mathbf{x};\boldsymbol{\omega}) \, d\boldsymbol{\omega} \ \Bigg| \  \|f\|_\rho :=\sup_{\boldsymbol{\omega}} \left| \frac{\alpha(\boldsymbol{\omega})}{\rho(\boldsymbol{\omega})} \right|<\infty \right\},
\end{equation*}
where $\rho(\boldsymbol{\omega})= \rho(\omega_1) \ldots \rho(\omega_d)$.
The random feature approximation of $f\in \mathcal{F}(\phi,\rho)$ is denoted by $f^\sharp$ and defined as
 \begin{equation}\label{eq:linearcombo}
    f^\sharp(\mathbf{x}) = \sum_{j=1}^N c^\sharp_j \,\phi(\mathbf{x},\boldsymbol{\omega}_j),
\end{equation}
where the weights $\{\boldsymbol{\omega}_j\}_{j\in [N]}$ are sampled i.i.d. from the density $\rho$.
Following \cite{rahimi2007random, rahimi2008weighted, rahimi2008uniform}, the best $\phi$-based approximation of $f\in  \mathcal{F}(\phi,\rho)$ is given by $f^\star$
\begin{equation}\label{eq:fstar}
    f^{\star}(\mathbf{x}) =  \frac{1}{N} \sum_{j=1}^N \frac{\alpha(\boldsymbol{\omega}_j)}{\rho(\boldsymbol{\omega}_j)}\, \phi(\mathbf{x},\boldsymbol{\omega}_j),
\end{equation}
where the coefficients with respect to the random features are defined as $c^{\star}_j:= \frac{\alpha(\boldsymbol{\omega}_j)}{N\rho(\boldsymbol{\omega}_j)}$ for all $j\in[N]$ and thus $\|c^\star\|_2 \leq N^{-\frac{1}{2}} \, \|f\|_\rho$. For example, let $\rho$ be the density associated with $\mathcal{N}({0},\sigma^2)$ and assume that the conditions of Theorem~\ref{thm:convergence} hold. In this setting, it was shown in \cite{hashemi2021generalization} that $\|\mathbf{e}^\star\|_\infty=\|A\mathbf{c}^\star-\mathbf{y}\|_\infty \leq \epsilon \|f\|_\rho$, where
\begin{align*}
\epsilon:=  \frac{1}{\sqrt{N}}\left(1 + 4 \gamma\sigma  d \sqrt{1+\sqrt{\frac{12}{d} \log\frac{m}{\delta}}}+ \sqrt{\frac{1}{2}\log\left(\frac{1}{\delta}\right)}\right).
\end{align*}
Therefore, the bound in Theorem~\ref{thm:convergence} becomes
\begin{align}\label{eqn:function_bound}
\| \c^{n}-\c^\star\|_2 
&\leq \left(2\beta^n \,N^{-\frac{1}{2}}+ \frac{C}{\sqrt{s}} \left(1-sN^{-1}\right)  + D\, \sqrt{\frac{\epsilon^2 +\lambda N^{-1}}{1+\lambda}} \right) \, \|f\|_\rho,
\end{align}
which scales like $N^{-\frac{1}{2}}$. Equation~\eqref{eqn:function_bound} could be refined, since in multiple places an $\ell^2$ bound is replaced by an $\ell^\infty$ norm (noting that $\|f\|_\rho$ is essentially an infinity-like norm). 

For the HARFE results in the following two sections, we observed that the value of $\mu$ does not have a significant impact on the generalization. Therefore, we set $\mu=0.1$ for all experiments, which can be shown to produce a convergent sequence by extending the proofs found in \cite{foucart2011hard}.  In addition, we fix the sparsity ratio $(s/N)$ to be $5\%$ or $10\%$.  Since the parameter $\lambda$ depends on the dataset and the noise level, it should be tuned, for example, using cross-validation. In our experiments, we found that the optimal regularization parameter can range from $10^{-12}$ to $10^{-1}$ depending on the input data (since the data is not normalized). Thus, we optimize our results over a set of possible values for $\lambda$ in order to obtain good generalization. 
\section{Numerical Results on Synthetic Data}
In this section, we test Algorithm \ref{alg:cap} for approximating sparse additive functions including the benchmark examples discussed in \cite{meyer2003support, beylkin2009multivariate, binev2011fast
,potts2021interpretable,hashemi2021generalization}. The experiments show that the HARFE model outperforms several existing methods in terms of testing errors. 

The step-size parameter is set to $\mu=0.1$ for all experiments. Other hyperparameters will be specified for each experiment. The relative test error is calculated as
\begin{equation}
    \text{Rel}(f,f^{\sharp}) = \sqrt{\dfrac{\sum\limits_{\mathbf{x}\in X_{\text{test}}}|f(\mathbf{x}) - f^\sharp(\mathbf{x})|^2}{\sum\limits_{\mathbf{x}\in X_{\text{test}}}|f(\mathbf{x})|^2}},
\end{equation}
and the mean-squared test error is defined as
\begin{equation}
    \text{MSE}(f,f^\sharp) = \dfrac{1}{|X_{\text{test}}|}\sum\limits_{\mathbf{x}\in X_{\text{test}}}|f(\mathbf{x}) - f^\sharp(\mathbf{x})|^2
    \label{eqn:MSE}
\end{equation}
where $f$ is the target function and $f^\sharp$ is the trained function.

\subsection{Low-Order Function Approximation}
In the first example, we show the advantage of using a greedy approach over an $\ell^1$ optimization problem and the benefit of the additional ridge term. The input data is sampled from a uniform distribution $\mathcal{U}[-1,1]^d$ and the activation function is set to $\phi(\cdot)=\sin(\cdot)$. The nonzero entries of the random weights $\boldsymbol{\omega}_j$ are sampled from $\mathcal{N}(0,1)$. We introduce a set of random bias terms $b_j\in \R$ for $j\in[N]$ so that the random feature matrix is now defined as $a_{k,j} = \sin(\langle \mathbf{x}_k,\boldsymbol{\omega}_j \rangle + b_j)$. The bias is sampled from  $\mathcal{U}[0,2\pi]$ to cover all phase angles. The number of random weights is $N =10^4$, the sparsity level is $s=500$, and the number of training and testing data are  $m_\text{train}=500$ and $m_\text{test} = 500$, respectively. In all experiments in this section, we set the maximal number of iterations for our method to $50$. 

Table \ref{table:srfe} shows the median relative error (as a percentage) over 10 randomly generated test sets. In the experiments, we compared the results using $q\in \{1,3,5\}$. In each case, the HARFE approach with $\lambda>0$ is more accurate than HARFE with $\lambda=0$ and the SRFE \cite{hashemi2021generalization}. We observe that when the exact order $q$ is known, the error is lower (see Column 5 of Table \ref{table:srfe}). Although not included in the table, it is worth mentioning that the Elastic Net model performs comparably to the SRFE model, although it does not have the same generalization theory \cite{hashemi2021generalization}.

\begin{table}

\begin{center}
\begin{tabular}{|l|c|c|c|c|c|}
 \hline
  & $q$ & $\dfrac{1}{\sqrt{1+\|\textbf{x}\|_2^2}}$ &$\sqrt{1+\|\textbf{x}\|_2^2}$ &  $ \dfrac{x_1 x_2}{1+x_3^6}$& $\sum\limits_{i=1}^d \exp(-|x_i|)$  \\
\hline
 $d$&  & 5 & 5 &  5 &100 \\
 \hline
 \hline
SRFE & 1 & 3.30 & 1.29 &  102.1  & \color{blue}1.20 \\
 \hline
 HARFE, $\lambda=0$ & 1&3.30 & 1.40 & 105.6     &1.50 \\
\hline
 HARFE, $\lambda>0$ & 1 & 3.20 & 1.00    & 100 &{\color{blue} 1.10} \\
\hline
\hline
 SRFE & 3 & 0.80 &1.0   &8.0 & 1.80\\
\hline
HARFE, $\lambda=0$ & 3 & 1.00 &  \color{blue}0.25    &\color{blue}{3.20}&2.70 \\
\hline
HARFE, $\lambda>0$ & 3 &0.73  & \color{blue}{0.18}     &\color{blue}{3.40} & 2.01  \\
\hline
\hline
SRFE & 5 & \color{blue}{0.56} & 1.10  & 9.24&2.04 \\
\hline
HARFE, $\lambda=0$ & 5&1.80 & 1.08   & 11.20 &3.00 \\
\hline
HARFE, $\lambda>0$ &5 &\color{blue}{0.57} & 1.00   & 7.70 
&2.20 \\
\hline
\end{tabular}

\end{center}
\caption{Relative test errors (as percentage) for approximating various nonlinear functions using different $q$ values. For each function, the two smallest errors are highlighted (in blue). The ridge paramater $\lambda$ using in the HARFE  approach is set to $10^{-4}, 10^{-10}, 10^{-10}$, and $10^{-1}$ (going left to right). In all experiments, $m_{train} =m_{test} = 500$, $N=10^4$, $\x\sim \mathcal{U}[-1,1]^d$, the nonzero entries of $\boldsymbol{\omega}$ are drawn from $\mathcal{N}(0,1)$ and bias is drawn from $\mathcal{U}[0,2\pi]$.} 
\label{table:srfe}
\end{table}

\subsection{Approximation of Friedman Functions}
In this example, we test the HARFE method on the  Friedmann functions, which are used as benchmark examples for certain approximation techniques \cite{meyer2003support, beylkin2009multivariate, binev2011fast
,potts2021interpretable}. The three Friedman functions are defined as, $f_1:[0,1]^{10}\rightarrow \R$
   \[f_1(x_1,...,x_{10}) = 10\sin(\pi x_1x_2)+20\left(x_3-\frac{1}{2}\right)^2+10x_4+5x_5,\]
    $f_2: [0,1]^4\rightarrow \R$
   \[f_2(x_1,x_2,x_3,x_4) = 
 \sqrt{(100x_1)^2 + \Big(x_3(520\pi x_2 + 40\pi) - \dfrac{1}{(520\pi x_2 + 
  \pi)(10x_4 + 1)}\Big)^2},\]
  and  $f_3:[0,1]^4\rightarrow\R$
  \[f_3(x_1,x_2,x_3,x_4) = \arctan\left(\dfrac{x_3(520\pi x_2 + 40\pi) - (520\pi x_2 + 40\pi)^{-1}(10x_4 + 1)^{-1}}{100x_1}\right).\]

\begin{table}[t]
\begin{center}
\begin{tabular}{|l|c|c|c|c|}
 \hline
 Method & $f_1$ & $f_2$ ($\times10^3$)  & $f_3$ ($\times10^{-3}$) \\
 \hline
 svm &  4.36 & 18.13 & 23.15\\
 \hline
 lm &  7.71 & 36.15 & 45.42 \\
 \hline
 mnet &  9.21 & 19.61 & 18.12\\
 \hline
 rForst &  6.02& 21.50& 22.21\\
 \hline
 ANOVA &  \color{blue}{1.43} &17.21 &20.69 \\
\hline
 SRFE, $q=2$ &  2.35 & 5.15 & 18.88\\
\hline
HARFE, $q=2$ &  \color{blue}{1.52}   & \color{blue}{1.31}  
& \color{blue}{10.90} \\
\hline
HARFE, $q$   &  3.01
& \color{blue}{1.90}  & \color{blue}{13.28} \\
    \hline
\end{tabular}

\end{center}
\caption{Mean-squared test errors of different methods when approximating Friedman functions. The values for SRFE and HARFE are obtained by training the model on 100 randomly generated training sets and validating them on 100 randomly generated test sets. We test for different $q$ values. In the last row, $q=5$ for $f_1$ and $q=d=4$ for $f_2$ and $f_3$. For the HARFE, $\lambda=1 \times 10^{-3}, 5\times 10^{-3},$ and $1 \times 10^{-5}$ for the functions $f_1,f_2,$ and $f_3$, respectively. The two best values for every function are highlighted in blue.}
\label{table:friedman}
\end{table}

We follow the setup from \cite{potts2021interpretable}, where both the training and testing input datasets are randomly generated from the uniform distribution $\mathcal{U}[0,1]^d$, $m_{train} = 200$, and $m_{\text{test}}=1000$. In addition, Gaussian noise with zero mean and standard deviations of $\sigma = 1.0,125.0,$ and $0.1$, are added to the output data. The MSE of various methods when approximating Friedman functions are displayed in Table \ref{table:friedman}. We include the results of various methods found in \cite{potts2021interpretable} and compare against the results of SRFE \cite{hashemi2021generalization} and HARFE. For HARFE, the nonzero entries of the random weight vectors $\boldsymbol{\omega}$ and the bias terms are sampled from $\mathcal{U}[-1,1]$. We use $N=10^4$ features for $f_1$ and $N=2\times 10^3$ features for $f_2$ and $f_3$, $s=200$, and 50 iterations in total. Our proposed method achieves the smallest errors when approximating Friedman functions $f_2$ and $f_3$ even when the value of $q$ is unknown (in that case, we assign $q=d$). 

It is worth noting that although the first Friedman function has $d=10$, it is a sparse additive model of order-$2$, and thus is better approximated by the  ANOVA, SRFE, and HARFE models. This is verified by the results in Table \ref{table:friedman} with $q=2$. Note that HARFE with $q=2$ yields a comparable result with ANOVA for $f_1$ while outperforming ANOVA in the other two examples. For the second Friedman function, HARFE with $q=2$ is significantly better than the other methods by almost a factor of 14 times. For the third Friedman function, when the scale of the data is taken into consideration, all methods produce slightly worse results, with HARFE producing the lowest error overall. In Figure \ref{fig:my_label}, scatter plots show the true data compared to the predicted values using the HARFE model over a test set for functions $f_1$, $f_2$ and $f_3$. The HARFE model produces lower variances for $f_1$ and $f_2$, while some bias occurs in $f_3$ near zero.

\begin{figure}[h!]
    \centering
    \includegraphics[scale=0.42]{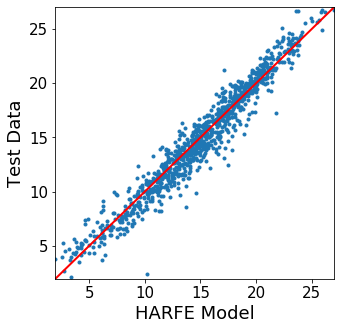}
    \includegraphics[scale = 0.42]{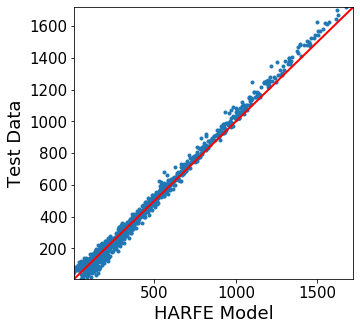}
    \includegraphics[scale = 0.42]{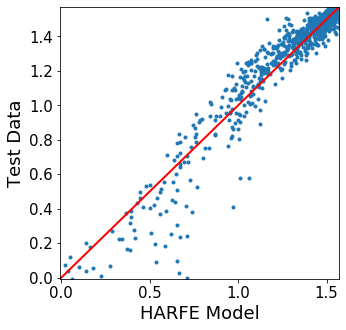}
    \caption{Scatter plots of the true data versus the predicted values using the HARFE model over the test set for functions $f_1$, $f_2$ and $f_3$ (from left to right). }
    \label{fig:my_label}
\end{figure}

\begin{figure}
    \centering
    \includegraphics[scale = 0.45]{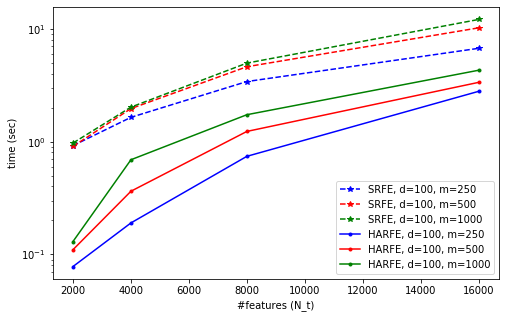}
    \includegraphics[scale = 0.45]{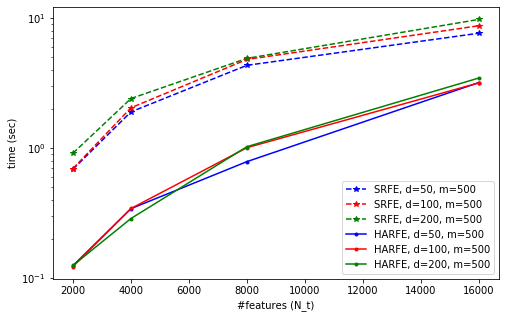}
    \caption{Plots showing the time required (in seconds) for optimizing the trainable weights $\mathbf{c}$ using HARFE compared with SRFE for $m\in\{250,500,1000\}$ with $d=100$ (left) and $d\in\{50,100,200\}$ with $m=500$ (right) for the function $f(\mathbf{x})=\sqrt{1+\|\mathbf{x}\|_2^2}$.}
    \label{fig:runtime}
\end{figure}
In Figure \ref{fig:runtime}, we plot the runtimes (in seconds) of both HARFE and SRFE during the training phase for $f(\mathbf{x})=\sqrt{1+\|\mathbf{x}\|_2^2}$ with different $m$ and $d$. The stopping criteria used for both the algorithm was same. We can see clearly that HARFE is almost 2.5 times faster than SRFE as the number of features $N$ increases.

In the next experiment,  we would like to test HARFE for feature selection. Figure \ref{fig:friedmanHist} displays a histogram illustrating the distribution of indices retained by the HARFE approach for the Friedman function $g:[0,1]^{20}\rightarrow \R,$
   \[g(x_1,...,x_{20}) = 10\sin(\pi x_1x_2)+20\left(x_3-\frac{1}{2}\right)^2+10x_4+5x_5.\]
   In this experiment, we choose $q=2$. From the histogram in Figure \ref{fig:friedmanHist}, we observe that HARFE can redistribute the weights based on active input variables especially when applied for a $q$-order additive function satisfying $q\ll d$ and the number of active variables (five in this case) is much less than the input dimension (which is twenty) of the function. Specifically, the histogram is based on the occurrence rate (as a percentage) of all twenty variables obtained from the HARFE model. The top 5 indices correspond exactly with the correct set.

   \begin{figure}[h!]
    \centering
    \includegraphics[scale=0.5]{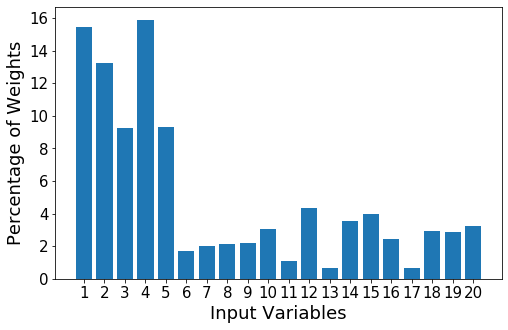}
    \caption{A histogram plot displaying the distribution of indices retained by the HARFE approach applied to the Friedman function $g(x_1,...,x_{20}) = 10\sin(\pi x_1x_2)+20\left(x_3-\frac{1}{2}\right)^2+10x_4+5x_5$ using $q=2$.} The histogram is based on the occurrence rate (as a percentage) of the input variables obtained from the HARFE model. The HARFE model correctly identifies the dominate index set.
    \label{fig:friedmanHist}
\end{figure}

\section{Numerical Results on Real Datasets}
We compare the performance of models obtained by the HARFE algorithm with other state-of-the-art sparse additive models \cite{liu2020sparse,kandasamy2016additive,hashemi2021generalization} when applied to {eleven} real datasets. {An overview of the datasets and the hyperparameters $s,q,m\lambda$ used in the HARFE model are presented in Table \ref{table:data}.}

\begin{table}[h]
\begin{center}
\begin{tabular}{|l|l|l|l|l|l|l|l|l| }
\hline
Dataset &dim & train & val & N  & $s$ & $m\lambda$ &  $q$   \\
\hline
Propulsion & 15 &200 & 200 &  3k & 300 & $10^{-10}$ &2  \\
 \hline
  Galaxy & 20 &2000 & 2000 & 10k & 1k &$10^{-7}$ &2\\
 \hline
  Skillcraft & 18 &1700 & 1630 &  20k & 1k & 1.0&2 \\
 \hline
 Airfoil &  41 &750 & 750 & 80k & 5k &1.0 & 2\\
 \hline
 Forestfire & 10 &211 & 167 &  4220 & 422 & 0.5&2\\
 \hline
  Housing & 12 &256 & 250 &  10k & 1k &0.1 &2 \\
  \hline
 Music & 90 &1000 & 1000 &  10k & 666 & 2.5&3\\
 \hline
Insulin & 50 & 256 & 250 &  2560 & 170 &2.75 &2\\
 \hline
 Speech & 21 &520 & 520 &  20k & 1k & 0.1 &2\\

 \hline
 Telemonitor & 19 &1000 & 867 &  15k & 937 & 0.1&5\\
 \hline
  CCPP & 59 &2000 & 2000 & 10k & 1k & 0.05&1 \\

 \hline
 \end{tabular} 
 \caption{Overview of eleven datasets { and the values of $s,m\lambda$, and $q$ used in the HARFE model}. The experimental setup and datasets for each test follow from \cite{Dua:2019, ravikumar2007spam, kandasamy2016additive, liu2020sparse}.}
 \label{table:data}
   \end{center}

\end{table}

The results of COSSO, Lasso, SALSA, SpAM, and SSAM are obtained from \cite{ravikumar2007spam, kandasamy2016additive, liu2020sparse} and we include the results of the SRFE and HARFE model. The experiments follow the setup from \cite{ravikumar2007spam, kandasamy2016additive, liu2020sparse}, where the training data is normalized so that the input and output values have zero mean and unit variance along each dimensions. Each dataset is divided in half to form the training and testing sets.  The results and comparisons are shown in Table \ref{table:realdata}. {The HARFE approach produces the lowest errors on eight datasets and achieves comparable results on the remaining datasets. Specifically, we significantly outperform other methods on the Propulsion, Airfoil and Forestfire datasets.}  
\begin{table}[h!]
\begin{center}
    \begin{tabular}{|l|c|c|c|c|c|c|c|}
    \hline 
    & HARFE & COSSO & Lasso & SALSA & SpAM &   SRFE &SSAM\\
    \hline
   
         Propulsion & \color{blue}
         0.0000417 & 0.00094 & 0.0248 & 0.0088 & 1.1121 & 0.0154 & - \\
           \hline
            Galaxy & \color{blue}0.0001024 & 0.00153 & 0.0239 & 0.00014 & 0.9542 & 0.0012 &-\\
 \hline
         Skillcraft & \color{blue}0.5368 & 0.5551 & 0.6650 & 0.5470 & 0.9055 &0.8730 &  0.5432\\
         \hline
         Airfoil &\color{blue}0.4492 & 0.5178 & 0.5199 & 0.5176 & 0.9623& 0.5702 & 0.4866 \\
         \hline 
         Forestfire&\color{blue} 0.2937 & 0.3753 & 0.5193& 0.3530 & 0.9694& 0.4067 & 0.3477 \\
         
         \hline 
          Housing & \color{blue} 0.2636 & 1.3097 & 0.4452 &  0.2642 & 0.8165 & 0.6395 & 0.3787\\
          \hline 
         Music& \textcolor{blue}{0.6134} & 0.7982& 0.6349 & 0.6251 & 0.7683 & {1.0454} & 0.6295\\
         \hline 
          Insulin & \textcolor{blue}{1.0137} & 1.1379 & 1.1103 & 1.0206 & 1.2035 & {1.6456} & 1.0146 \\
       \hline 
         Speech & 0.0238&  0.3486 & 0.0730& \color{blue} 0.0224 & 0.6600 & 0.0246 & - \\
        
        \hline
       
         Telemonitor & {0.0370} & 5.7192 & 0.0863 & 0.0347& 0.8643& \color{blue} 0.0336 & 0.0689 \\
         \hline
         CCPP &0.0677  &0.9684&0.07395& 0.0678 &\color{blue} 0.0647& 0.07440&0.0694\\
         \hline 
    \end{tabular}
    \caption{Average MSE on real datasets using various sparse additive models including COSSO, Lasso, SALSA, SpAM, SRFE, SSAM, and HARFE. The lowest error for each dataset is highlighted in blue.}
    \label{table:realdata}
\end{center}
\end{table}

\begin{figure}[h!]
  
     \begin{subfigure}[b]{0.5\textwidth}
         \centering
         \includegraphics[scale = 0.4]{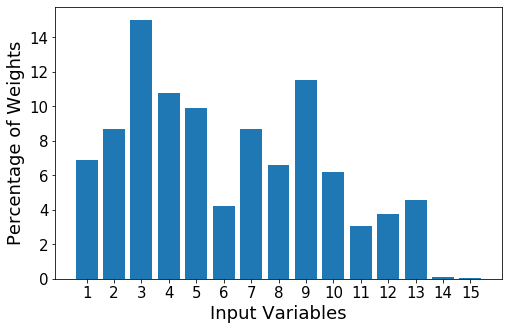}
         \caption{Propulsion dataset}
         \label{fig:Histogram_Plotsa}
     \end{subfigure}
     \hfill
     \begin{subfigure}[b]{0.5\textwidth}
         \centering
         \includegraphics[scale = 0.4]{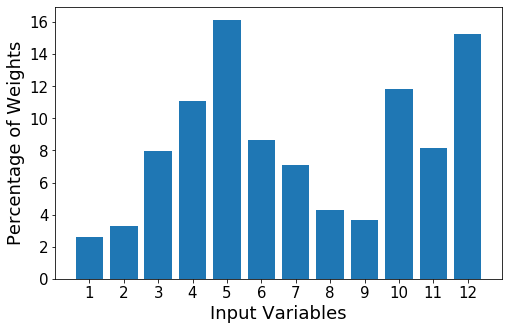}
         \caption{Housing dataset}
         \label{fig:Histogram_Plotsb}
     \end{subfigure}
    \\
     
     \begin{subfigure}[b]{0.5\textwidth}
         \centering
         \includegraphics[scale = 0.4]{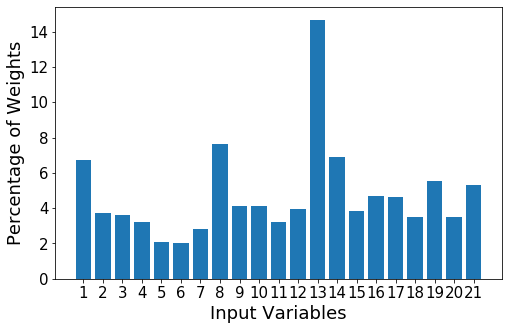}
         \caption{Speech dataset}
         \label{fig:Histogram_Plotsc}
     \end{subfigure}
     \hfill
     \begin{subfigure}[b]{0.5\textwidth}
         \centering
         \includegraphics[scale = 0.4]{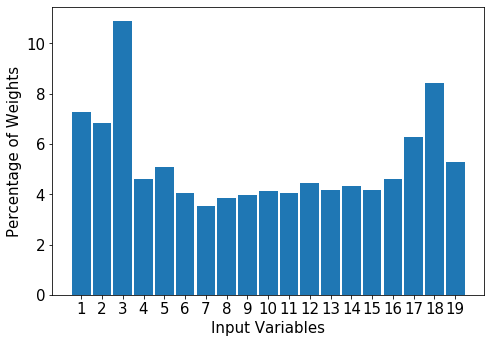}
         \caption{Telemonitoring dataset}
         \label{fig:Histogram_Plotsd}
     \end{subfigure}
     \caption{ The histogram plots the percentage of weights corresponding to each variable based on the (sparse) coefficient vector approximated using HARFE for the Propulsion, Housing, Speech, and the Telemonitor datasets, respectively. }
    \label{fig:Histogram_Plots}
\end{figure}

In Figure \ref{fig:Histogram_Plots}, we plot the histogram of the percentage of weights corresponding to each variable based on the (sparse) coefficient vector approximated using HARFE (from Table \ref{table:realdata}) for the Propulsion, Housing, Speech, and Telemonitor datasets, respectively. For the Propulsion test in Figure \ref{fig:Histogram_Plotsa}, we see that the 14th variable (Gas turbine compressor decay state coefficient) and the 15th variable (Gas turbine decay state coefficient) have the least contribution to the predictor (Lever Position), while the 3rd variable (Gas turbine rate of revolutions) is the most relevant variable to the predictor. From the random sampling, the histograms are initiated uniformly  (at least in expectation). This experiment shows that the HARFE algorithm will redistribute the weights and identify important variables as a benefit of the sparsity-promoting aspect.

From Figure \ref{fig:Histogram_Plotsb}, the HARFE variable selection suggests that the predictor of the House dataset (per capita crime rate by town) is most affected by the 5th variable (proportion of owner-occupied units built prior to 1940) and the 12th variable (median value of owner-occupied homes in $\$1000$'s). In Figure \ref{fig:Histogram_Plotsc}, the plot shows that the 13th (Noise-to-Harmonic or NTH parameter) significantly contributes to the output of the predictor (median pitch) of speech dataset. Lastly, for the Telemonitoring, the experiment in Figure \ref{fig:Histogram_Plotsd} shows several significant contributors to HARFE trained predictor.

\section{Summary}

In this work, we proposed a new high-dimensional sparse additive model utilizing the random feature method with two sparsity priors. First, we assumed that the number of terms needed in the model is small which leads to function approximations with low model complexity. Secondly, we enforce a random and sparse connectivity pattern between the hidden layer and the input layer which helps to extract input variable dependencies. Based on the numerical experiments on high-dimensional synthetic examples, the Friedman functions, and real data, the HARFE algorithm was shown to produce robust results that have the added benefit of extracting interpretable variable information. The analysis of the HARFE algorithm utilizes techniques from compressive sensing and it was shown that the method converges and has a reasonable error bound depending on the number of features, the number of samples, the ridge parameter, the sparsity, the noise level, and dimensional parameters.  We expect that risk bounds for the HARFE model can be obtained by following the proofs found in \cite{hashemi2021generalization}. In ongoing work, we would like to incorporate prior variable dependency information within the construction of the weight matrix. 
\section*{Acknowledgement}
E.S. and G.T. were supported in part by NSERC RGPIN 50503-10842. H.S. was supported in part by AFOSR MURI FA9550-21-1-0084 and NSF DMS-1752116.

 \section*{Conflict of Interest Statement}
 
On behalf of all authors, the corresponding author states that there is no conflict of interest.

\bibliographystyle{abbrv}
%\bibliography{references}{}

\begin{thebibliography}{10}

\bibitem{bach2017equivalence}
F.~Bach.
\newblock On the equivalence between kernel quadrature rules and random feature
  expansions.
\newblock {\em The Journal of Machine Learning Research}, 18(1):714--751, 2017.

\bibitem{bach2008consistency}
F.~R. Bach.
\newblock Consistency of the group lasso and multiple kernel learning.
\newblock {\em Journal of Machine Learning Research}, 9(6), 2008.

\bibitem{bartlett2020benign}
P.~L. Bartlett, P.~M. Long, G.~Lugosi, and A.~Tsigler.
\newblock Benign overfitting in linear regression.
\newblock {\em Proceedings of the National Academy of Sciences},
  117(48):30063--30070, 2020.

\bibitem{belkin2019reconciling}
M.~Belkin, D.~Hsu, S.~Ma, and S.~Mandal.
\newblock Reconciling modern machine-learning practice and the classical
  bias--variance trade-off.
\newblock {\em Proceedings of the National Academy of Sciences},
  116(32):15849--15854, 2019.

\bibitem{belkin2019does}
M.~Belkin, A.~Rakhlin, and A.~B. Tsybakov.
\newblock Does data interpolation contradict statistical optimality?
\newblock In {\em The 22nd International Conference on Artificial Intelligence
  and Statistics}, pages 1611--1619. PMLR, 2019.

\bibitem{bertsimas2020sparse}
D.~Bertsimas and B.~Van~Parys.
\newblock Sparse high-dimensional regression: Exact scalable algorithms and
  phase transitions.
\newblock {\em The Annals of Statistics}, 48(1):300--323, 2020.

\bibitem{beylkin2009multivariate}
G.~Beylkin, J.~Garcke, and M.~J. Mohlenkamp.
\newblock Multivariate regression and machine learning with sums of separable
  functions.
\newblock {\em SIAM Journal on Scientific Computing}, 31(3):1840--1857, 2009.

\bibitem{binev2011fast}
P.~Binev, W.~Dahmen, and P.~Lamby.
\newblock Fast high-dimensional approximation with sparse occupancy trees.
\newblock {\em Journal of Computational and Applied Mathematics},
  235(8):2063--2076, 2011.

\bibitem{iht}
T.~Blumensath and M.~E. Davies.
\newblock Iterative hard thresholding for compressed sensing.
\newblock {\em Applied and Computational Harmonic Analysis}, 27(3):265--274,
  2009.

\bibitem{htp2}
J.-L. Bouchot, S.~Foucart, and P.~Hitczenko.
\newblock Hard thresholding pursuit algorithms: number of iterations.
\newblock {\em Applied and Computational Harmonic Analysis}, 41(2):412--435,
  2016.

\bibitem{campbell2002kernel}
C.~Campbell.
\newblock Kernel methods: a survey of current techniques.
\newblock {\em Neurocomputing}, 48(1-4):63--84, 2002.

\bibitem{chen2021conditioning}
Z.~Chen and H.~Schaeffer.
\newblock Conditioning of random feature matrices: Double descent and
  generalization error.
\newblock {\em arXiv preprint arXiv:2110.11477}, 2021.

\bibitem{christmann2016learning}
A.~Christmann and D.-X. Zhou.
\newblock Learning rates for the risk of kernel-based quantile regression
  estimators in additive models.
\newblock {\em Analysis and Applications}, 14(03):449--477, 2016.

\bibitem{Dua:2019}
D.~Dua and C.~Graff.
\newblock {UCI} machine learning repository, 2017.

\bibitem{weinan2020towards}
W.~E, C.~Ma, S.~Wojtowytsch, and L.~Wu.
\newblock Towards a mathematical understanding of neural network-based machine
  learning: what we know and what we don't.
\newblock {\em arXiv preprint arXiv:2009.10713}, 2020.

\bibitem{foucart2011hard}
S.~Foucart.
\newblock Hard thresholding pursuit: an algorithm for compressive sensing.
\newblock {\em SIAM Journal on Numerical Analysis}, 49(6):2543--2563, 2011.

\bibitem{foucart2013mathematical}
S.~Foucart and H.~Rauhut.
\newblock {\em A Mathematical Introduction to Compressive Sensing}.
\newblock Birkh{\"a}user Basel, 2013.

\bibitem{frankle2018lottery}
J.~Frankle and M.~Carbin.
\newblock The lottery ticket hypothesis: Finding sparse, trainable neural
  networks.
\newblock {\em arXiv preprint arXiv:1803.03635}, 2018.

\bibitem{gonen2011multiple}
M.~G{\"o}nen and E.~Alpayd{\i}n.
\newblock Multiple kernel learning algorithms.
\newblock {\em The Journal of Machine Learning Research}, 12:2211--2268, 2011.

\bibitem{harris2019additive}
K.~D. Harris.
\newblock Additive function approximation in the brain.
\newblock {\em arXiv preprint arXiv:1909.02603}, 2019.

\bibitem{hashemi2021generalization}
A.~Hashemi, H.~Schaeffer, R.~Shi, U.~Topcu, G.~Tran, and R.~Ward.
\newblock Generalization bounds for sparse random feature expansions.
\newblock {\em arXiv preprint arXiv:2103.03191}, 2021.

\bibitem{hastie2019surprises}
T.~Hastie, A.~Montanari, S.~Rosset, and R.~J. Tibshirani.
\newblock Surprises in high-dimensional ridgeless least squares interpolation.
\newblock {\em arXiv preprint arXiv:1903.08560}, 2019.

\bibitem{hazimeh2020fast}
H.~Hazimeh and R.~Mazumder.
\newblock Fast best subset selection: Coordinate descent and local
  combinatorial optimization algorithms.
\newblock {\em Operations Research}, 68(5):1517--1537, 2020.

\bibitem{hearst1998support}
M.~A. Hearst, S.~T. Dumais, E.~Osuna, J.~Platt, and B.~Scholkopf.
\newblock Support vector machines.
\newblock {\em IEEE Intelligent Systems and their Applications}, 13(4):18--28,
  1998.

\bibitem{kandasamy2016additive}
K.~Kandasamy and Y.~Yu.
\newblock Additive approximations in high dimensional nonparametric regression
  via the salsa.
\newblock In {\em International Conference on Machine Learning}, pages 69--78.
  PMLR, 2016.

\bibitem{li2019towards}
Z.~Li, J.-F. Ton, D.~Oglic, and D.~Sejdinovic.
\newblock Towards a unified analysis of random {F}ourier features.
\newblock In {\em International Conference on Machine Learning}, pages
  3905--3914. PMLR, 2019.

\bibitem{liang2020just}
T.~Liang and A.~Rakhlin.
\newblock Just interpolate: Kernel ``ridgeless'' regression can generalize.
\newblock {\em The Annals of Statistics}, 48(3):1329--1347, 2020.

\bibitem{lin2006component}
Y.~Lin and H.~H. Zhang.
\newblock Component selection and smoothing in multivariate nonparametric
  regression.
\newblock {\em The Annals of Statistics}, 34(5):2272--2297, 2006.

\bibitem{liu2020sparse}
G.~Liu, H.~Chen, and H.~Huang.
\newblock Sparse shrunk additive models.
\newblock In {\em International Conference on Machine Learning}, pages
  6194--6204. PMLR, 2020.

\bibitem{luedtke2014branch}
J.~Luedtke.
\newblock A branch-and-cut decomposition algorithm for solving
  chance-constrained mathematical programs with finite support.
\newblock {\em Mathematical Programming}, 146(1):219--244, 2014.

\bibitem{maleki2009coherence}
A.~Maleki.
\newblock Coherence analysis of iterative thresholding algorithms.
\newblock In {\em 2009 47th Annual Allerton Conference on Communication,
  Control, and Computing (Allerton)}, pages 236--243. IEEE, 2009.

\bibitem{mazumder2017subset}
R.~Mazumder, P.~Radchenko, and A.~Dedieu.
\newblock Subset selection with shrinkage: Sparse linear modeling when the snr
  is low.
\newblock {\em arXiv preprint arXiv:1708.03288}, 2017.

\bibitem{mei2021generalization}
S.~Mei, T.~Misiakiewicz, and A.~Montanari.
\newblock Generalization error of random features and kernel methods:
  hypercontractivity and kernel matrix concentration.
\newblock {\em arXiv preprint arXiv:2101.10588}, 2021.

\bibitem{mei2019generalization}
S.~Mei and A.~Montanari.
\newblock The generalization error of random features regression: Precise
  asymptotics and the double descent curve.
\newblock {\em Communications on Pure and Applied Mathematics}, 2019.

\bibitem{meyer2003support}
D.~Meyer, F.~Leisch, and K.~Hornik.
\newblock The support vector machine under test.
\newblock {\em Neurocomputing}, 55(1-2):169--186, 2003.

\bibitem{potts2019approximation}
D.~Potts and M.~Schmischke.
\newblock Approximation of high-dimensional periodic functions with
  {F}ourier-based methods.
\newblock {\em arXiv preprint arXiv:1907.11412}, 2019.

\bibitem{potts2021interpretable}
D.~Potts and M.~Schmischke.
\newblock Interpretable approximation of high-dimensional data.
\newblock {\em SIAM Journal on Mathematics of Data Science}, 3(4):1301--1323,
  2021.

\bibitem{rahimi2007random}
A.~Rahimi and B.~Recht.
\newblock Random features for large-scale kernel machines.
\newblock In {\em NIPS}, volume~3, page~5. Citeseer, 2007.

\bibitem{rahimi2008uniform}
A.~Rahimi and B.~Recht.
\newblock Uniform approximation of functions with random bases.
\newblock In {\em 2008 46th Annual Allerton Conference on Communication,
  Control, and Computing}, pages 555--561. IEEE, 2008.

\bibitem{rahimi2008weighted}
A.~Rahimi and B.~Recht.
\newblock Weighted sums of random kitchen sinks: replacing minimization with
  randomization in learning.
\newblock In {\em NIPS}, pages 1313--1320. Citeseer, 2008.

\bibitem{ravikumar2007spam}
P.~Ravikumar, H.~Liu, J.~D. Lafferty, and L.~A. Wasserman.
\newblock Spam: Sparse additive models.
\newblock In {\em NIPS}, pages 1201--1208, 2007.

\bibitem{rudi2017generalization}
A.~Rudi and L.~Rosasco.
\newblock Generalization properties of learning with random features.
\newblock In {\em NIPS}, pages 3215--3225, 2017.

\bibitem{smola1998learning}
A.~J. Smola and B.~Sch{\"o}lkopf.
\newblock {\em Learning with kernels}, volume~4.
\newblock Citeseer, 1998.

\bibitem{tsigler2020benign}
A.~Tsigler and P.~L. Bartlett.
\newblock Benign overfitting in ridge regression.
\newblock {\em arXiv preprint arXiv:2009.14286}, 2020.

\bibitem{tyagi2016learning}
H.~Tyagi, A.~Kyrillidis, B.~G{\"a}rtner, and A.~Krause.
\newblock Learning sparse additive models with interactions in high dimensions.
\newblock In {\em Artificial Intelligence and Statistics}, pages 111--120.
  PMLR, 2016.

\bibitem{xie2020scalable}
W.~Xie and X.~Deng.
\newblock Scalable algorithms for the sparse ridge regression.
\newblock {\em SIAM Journal on Optimization}, 30(4):3359--3386, 2020.

\bibitem{xie2021shrimp}
Y.~Xie, B.~Shi, H.~Schaeffer, and R.~Ward.
\newblock Shrimp: Sparser random feature models via iterative magnitude
  pruning.
\newblock {\em arXiv preprint arXiv:2112.04002}, 2021.

\bibitem{xu2010simple}
Z.~Xu, R.~Jin, H.~Yang, I.~King, and M.~R. Lyu.
\newblock Simple and efficient multiple kernel learning by group lasso.
\newblock In {\em ICML}, 2010.

\bibitem{yen2014sparse}
I.~E.-H. Yen, T.-W. Lin, S.-D. Lin, P.~K. Ravikumar, and I.~S. Dhillon.
\newblock Sparse random feature algorithm as coordinate descent in {H}ilbert
  space.
\newblock In {\em Advances in Neural Information Processing Systems}, pages
  2456--2464, 2014.

\bibitem{zhang2005learning}
T.~Zhang.
\newblock Learning bounds for kernel regression using effective data
  dimensionality.
\newblock {\em Neural Computation}, 17(9):2077--2098, 2005.

\bibitem{zhou2019deconstructing}
H.~Zhou, J.~Lan, R.~Liu, and J.~Yosinski.
\newblock Deconstructing lottery tickets: Zeros, signs, and the supermask.
\newblock {\em arXiv preprint arXiv:1905.01067}, 2019.

\end{thebibliography}

\newpage
\appendix
\appendixpage
\section{Key Results}

\begin{theorem}[\textbf{Restricted Isometry Constants} from \cite{chen2021conditioning}]\label{thm:RIP}
Let the data $\{\mathbf{x}_k\}_{k\in [m]}$ be drawn from $\mathcal{N}(\mathbf{0},\gamma^2 \mathbf{I}_d)$, the weights $\{\boldsymbol{\omega}_j\}_{j\in [N]}$ be drawn from $\mathcal{N}(\mathbf{0},\sigma^2 \mathbf{I}_d)$, and the random feature matrix $\mathbf{A}\in\mathbb{C}^{m\times N}$ be defined component-wise by $a_{k,j} = \exp(i \langle \mathbf{x}_k,\boldsymbol{\omega}_j \rangle)$. For $\eta_1, \eta_2, \delta \in (0,1)$ and some integer $s\geq 1$, if
\begin{align*}
    m&\geq C_1 \eta_1^{-2}\, s\, \log(\delta^{-1}) \\
    \frac{m}{\log(3m)} &\geq C_2 \eta_2^{-2}\, s \,\log^2(s) \log\left(\frac{N}{9\log(2m)}+3\right) \\
    \sqrt{\delta}&\,\eta_1\, (4\gamma^2\sigma^2+1)^{\frac{d}{4}}\geq N,
\end{align*}
where $C_1$ and $C_2$ are universal positive constants, then with probability at least $1-2\delta$, the $s$ restricted isometry constant of $\frac{1}{\sqrt{m}}\mathbf{A}$ is bounded by
$$
\delta_{s}\left(\frac{1}{\sqrt{m}}\mathbf{A}\right)< 3\eta_1 + \eta_2^2 + \sqrt{2}\eta_2.
$$
\end{theorem}

\begin{theorem}[\textbf{Convergence of HTP} Theorem 6.20 from \cite{foucart2013mathematical}]\label{thm:convergenceHTP}
Suppose that the $(6s)^\text{th}$ order restricted isometry constant of $\mathbf{A}\in \C^{m \times N}$ satisfies $\delta_{6s}<\frac{1}{\sqrt{3}}$,
then for any $\c\in \C^N$ and $\mathbf{e}\in \C^m$, the sequence $\c^{k}$ defined by the hard thresholding pursuit with $y=\mathbf{A}\c+\mathbf{e}$, $\c^0=\textbf{0}$, using $2s$ instead of $s$ in the algorithm, satisfies
\begin{align}\label{eq:convhtpseq}
\| \c^{n}-\c\|_2 &\leq 2\beta^n \|\c\|_2 + \frac{D_1}{\sqrt{s}} \kappa_{1,s}(\c)+D_2 \|\mathbf{e}\|_2,
\end{align}
for all $n\geq 0$ where the constants $\beta \in(0,1)$, $D_1,D_2>0$ depend only on $\delta_{6s}$. In particular,
if the sequence $(\c^n)$ clusters around some $\c^{\sharp} \in \C^N$ , then
\begin{align}\label{eq:convhtpcluster}
\| \c - \c^{\sharp}\|_2 &\leq \frac{D_1}{\sqrt{s}} \kappa_{1,s}(\c)+D_2 \|\mathbf{e}\|_2.
\end{align}
\end{theorem}

\begin{lemma}[Lemma 1 in \cite{hashemi2021generalization}]\label{thm:lemma1} Fix the confidence parameter $\delta > 0$ and accuracy parameter $\epsilon > 0$. Suppose $f \in \mathcal{F}(\phi,\rho)$ where $\phi(\x,\boldsymbol{\omega}) = \exp(i\langle\x,\boldsymbol{\omega}\rangle)$ and $\rho$ is a probability distribution with finite second moment used for sampling the random weights $\boldsymbol{\omega}$. Suppose
$N\geq \dfrac{1}{\epsilon^2}\left(1+\sqrt{2\log\left(\dfrac{1}{\delta}\right)}\right)^2$, then with probability at least $1-\delta$, the following holds with respect to the draw of $\boldsymbol{\omega}_j$ for $j \in [N]$, 

\begin{equation}\label{eq:lemma1}
\|f-f^{\star}\|_{L^2(d\mu)}\leq \epsilon \|f\|_{\rho}
\end{equation}
where $f^{\star} (\x)= \sum\limits_{i=1}^N \mathbf{c}_k^{\star}\exp(i\langle\x,\boldsymbol{\omega}\rangle)$, with $\c_k^{\star} = \dfrac{\alpha(\boldsymbol{\omega}_k)}{N\rho(\boldsymbol{\omega}_k)}$.
\end{lemma}

\end{document}